\documentclass{article}

\usepackage[top=1in, bottom=1in, left=1in, right=1in]{geometry}

\usepackage{placeins}

\usepackage[utf8]{inputenc} % allow utf-8 input
\usepackage[T1]{fontenc}    % use 8-bit T1 fonts
\usepackage{hyperref}       % hyperlinks
\usepackage{url}            % simple URL typesetting
\usepackage{booktabs}       % professional-quality tables
\usepackage{nicefrac}       % compact symbols for 1/2, etc.
\usepackage{microtype}      % microtypography
\usepackage{xcolor}         % colors

\usepackage{natbib}

\usepackage{dsfont}
\usepackage{amsmath, amssymb, amsfonts, amsthm}
\usepackage{algorithm}
\usepackage[noend]{algpseudocode}
\usepackage{mathtools,commath}
\usepackage{graphicx, multirow}
\usepackage{hyperref}
\usepackage{subcaption, caption}

\usepackage[most]{tcolorbox}   % provides the colored box environment
\newtcolorbox{mistake}[1][]{
  colback=gray!5,          % background colour
  colframe=gray!80,        % border colour
  fonttitle=\bfseries,    % title font
  title=Prompt,            % default title (can be overridden)
  left=2mm, right=2mm, top=2mm, bottom=2mm,
  boxrule=0.8pt,
  breakable,
  rounded corners,          % use rounded corners if you prefer `rounded corners`
  #1                       % allow the user to pass extra options
}

\newtheorem{theorem}{Theorem}

\newtheorem{example}{Example}

\title{To Err Is Human: Systematic Quantification of Errors in Published AI Papers via LLM Analysis}

\makeatletter
\def\@fnsymbol#1{\ensuremath{\ifcase#1\or \or \or \or \or \or \or \or \or \else\@ctrerr\fi}}
\makeatother

\usepackage{titling}

\usepackage{abstract}

\begin{document}

\author{%
  Federico Bianchi$^{\scriptstyle \dagger}$\,$^{1}$, 
  Yongchan Kwon$^{\scriptstyle \dagger}$\,$^{1}$, 
  Zachary Izzo$^{\scriptstyle \dagger}$\,$^{2}$, 
  Linjun Zhang$^{3}$, 
  James Zou$^{1,4}$%
  \thanks{%
    $^{\scriptstyle \dagger}$Equal Contribution.
    $^{1}$Together AI,
    $^{2}$NEC Labs America,
    $^{3}$Rutgers University,
    $^{4}$Stanford University.
  }%
}
\date{}

{\hypersetup{hidelinks}\maketitle}

\vspace{-40pt}
\begin{abstract}
  How many mistakes do published AI papers contain? Peer-reviewed publications form the foundation upon which new research and knowledge are built. Errors that persist in the literature can propagate unnoticed, creating confusion in follow-up studies and complicating reproducibility. The accelerating pace of research and the increasing demands on the peer-review system make such mistakes harder to detect and avoid. To address this, we developed a Paper Correctness Checker based on GPT-5 to systematically identify mistakes in papers previously published at top AI conferences and journals. Our analysis focuses on objective mistakes—e.g., errors in formulas, derivations, calculations, figures, and tables—that have a clearly verifiable ground truth. We intentionally exclude subjective considerations such as novelty, importance, or writing quality. We find that published papers contain a non-negligible number of objective mistakes and that the average number of mistakes per paper has increased over time—from 3.8 in NeurIPS 2021 to 5.9 in NeurIPS 2025 ($55.3\%$ increase); from 4.1 in ICLR 2018 to 5.2 in ICLR 2025; and from 5.0 in TMLR 2022/23 to 5.5 in TMLR 2025. Human experts reviewed 316 potential mistakes identified by the AI Checker and confirmed that 263 were actual mistakes, corresponding to a precision of 83.2\%. While most identified issues are relatively minor, correcting them would reduce confusion in the literature and strengthen reproducibility. The AI Checker also surfaced potentially more substantive mistakes that could affect the interpretation of results. Moreover, we show that the AI Checker can propose correct fixes for 75.8\% of the identified mistakes. Overall, this study highlights the potential of frontier LLMs to detect and correct objective mistakes in published papers, helping to establish a firmer foundation of knowledge.  
\end{abstract}

\begin{figure}[!h]
    \centering
    \includegraphics[width=0.7\linewidth]{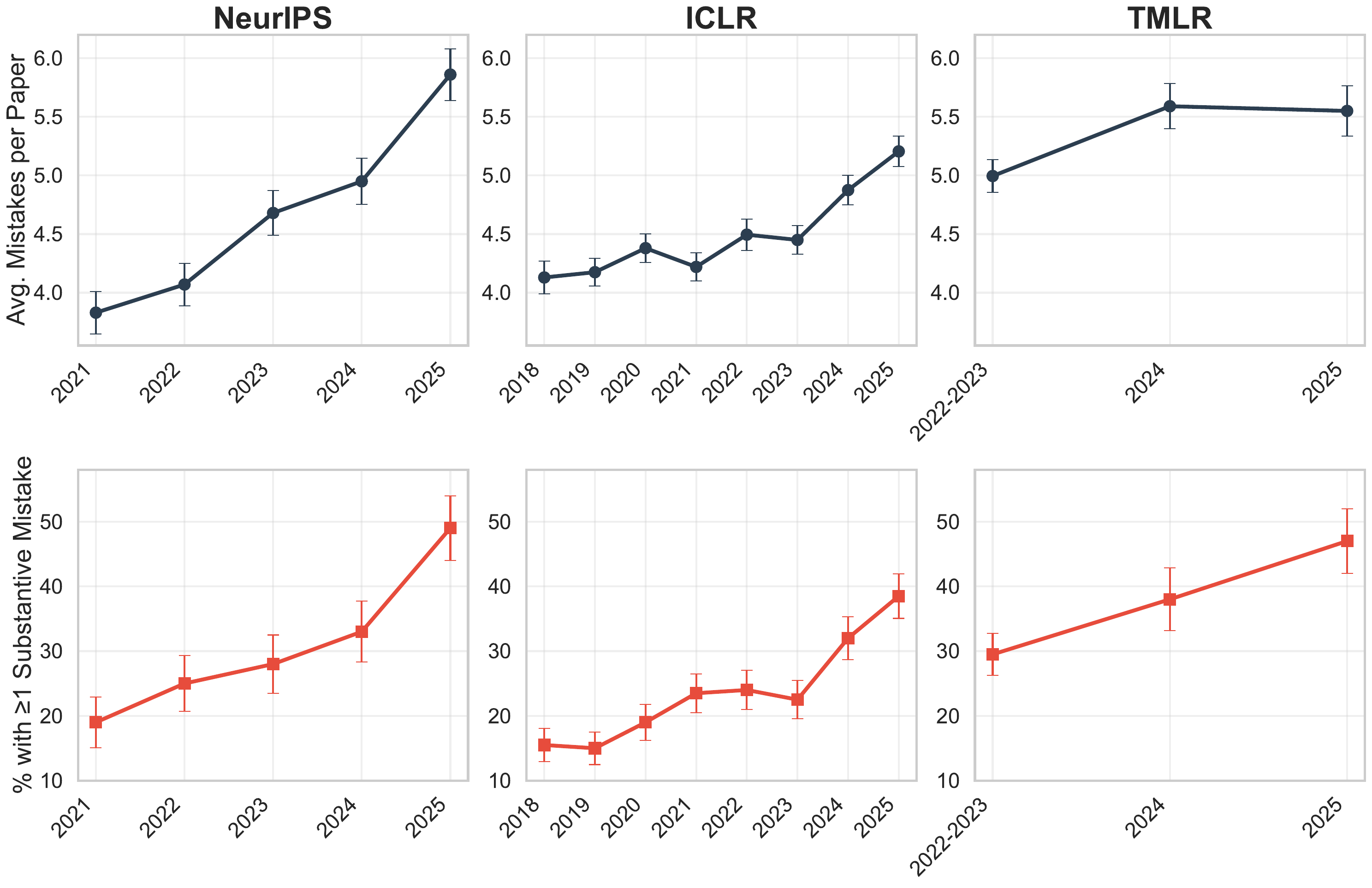}
    \vspace{-10pt}
    \caption{AI Checker detected mistakes for papers published in NeurIPS (left column), ICLR (middle), and TMLR (right) across years. The top row shows the average number of detected mistakes per paper, while the bottom row shows the percentage of papers with $\geq 1$ potentially substantive mistake. Error bars represent the standard errors.}
    \label{fig:error_rates}
\end{figure}

\section{Introduction}

The pace of publishing in AI  has accelerated dramatically. For instance, the number of submissions to ICLR grew from 1,013 in 2018 to 19,619 in 2026, reflecting an unprecedented expansion in both research activity and publication pressure \citep{papercopilot_iclr_stats_2025}. As the field scales, papers are often assembled rapidly in response to competitive deadlines, iterative experimentation, and the drive to disseminate results quickly.

Yet a modern AI paper is a complex, multi-component artifact. It typically weaves together mathematical derivations, algorithmic design, empirical evaluations, and claims about performance or methodology. With so many moving parts—and with authors frequently racing against time—there are numerous opportunities for mistakes to slip in. At the same time, peer reviewers face overwhelming workloads and shrinking bandwidth. Even conscientious reviewers may fail to identify subtle but important mistakes in logic, mathematics, or experimental reporting.

Revealing and correcting mistakes in previously published papers matters because scientific knowledge is cumulative: new models, benchmarks, and theoretical results often build directly on past work. When mistakes go unnoticed, they can propagate through follow-up research, distort empirical comparisons, or anchor entire lines of inquiry on incorrect assumptions. Published papers also guide real-world deployments and industry practices, amplifying the consequences of errors that remain embedded in the literature. By systematically identifying and documenting such mistakes, we gain a clearer picture of the reliability of our research foundations, strengthen reproducibility efforts, and highlight where peer-review processes may need reinforcement. 

In this work, we take a first step toward systematically quantifying the prevalence of  mistakes in published AI research. To enable consistent and verifiable assessment, we restrict our focus to objective errors—those that have a clear right-or-wrong answer—such as wrong calculations, incorrect formulas, logical contradictions, factual inaccuracies, and discrepancies between claims and reported results. We intentionally exclude subjective aspects of peer review, including judgments about novelty, significance, clarity, or writing quality, as well as surface issues like grammar and typos.

To carry out this analysis, we developed an AI Correctness Checker using GPT-5, designed to identify objective mistakes in research papers and suggest potential fixes. We validated the system by having human experts verify $316$ identified mistakes in published papers, demonstrating that the Correctness Checker has a precision of $83.2\%$. 

We then applied this LLM-based reviewer to $2,500$ randomly sampled papers published across three major machine learning venues: ICLR (2018–2025), NeurIPS (2021–2025), and TMLR (2022-2025). Across this corpus, the system identified an average of $4.7$ objective mistakes per paper and this number has increased over time. The most common issues were mathematical mistakes ($54.0\%$), including incorrect formulas, invalid derivations, and logical flaws in proofs or algorithmic reasoning.

\paragraph{Related work} 
Substantial human effort is currently devoted to identifying and correcting mistakes in published scientific literature. Platforms such as PubPeer\footnote{\url{https://pubpeer.com/}} and Retraction Watch\footnote{\url{https://retractionwatch.com/}} provide space for researchers to publicly document concerns and keep track of retracted papers. Journals also publish author-initiated corrections, editorial notes, or retraction notices when errors are discovered post-publication. Researchers have categorized the types of mistakes in publications~\cite{brewin2023inaccuracy,darling2024statistical}. Researchers have also investigated errors in machine learning experiments in a small scale assessment of 49 papers~\cite{shepperd2019prevalence}. All of these processes are performed manually, relying on the time, expertise, and vigilance of human reviewers.
Our work investigates whether frontier LLMs can assist with this effort by automatically detecting errors and suggesting possible fixes. Moreover, systematic and quantitative analysis of mistakes in peer-reviewed AI papers is lacking. Our study aims to fill this gap by combining large-scale automated analysis with human validation.

Prior work has explored LLMs as automated reviewers~\cite{liang2024can,zhou-etal-2024-llm,10.1016/j.inffus.2025.103332, xu2025can}, but these systems largely mirror human review criteria—clarity, novelty, significance, and overall writing quality. Some works have focused on detecting errors specifically in references~\cite{zhang2024detecting}. Only a few papers evaluated the ability of LLMs to detect errors~\cite{liu2023reviewergpt,xi2025flaws, li2024evaluating}. Overall the work on using LLMs to review papers is still nascent and tends to focus on holistic evaluation rather than quantifying objective mistakes. A concurrent work similar in spirit, albeit in a very different domain, uses LLMs to detect mathematical errors in 120 neurosurgery papers~\cite{ali2025}.  %As a result, the literature primarily addresses subjective aspects of evaluation rather than detecting objective errors such as mathematical inconsistencies, incorrect formulas, or factual contradictions. 
LLMs have been used to support the peer review process by providing feedback to human reviewers~\cite{Thakkar2025CanLF} and as general reviewing systems~\cite{bianchi2025exploring}. Existing efforts have also focused mostly on new manuscripts, aiming to assist authors or program committees during the submission process. In contrast, little work has examined whether LLMs can be used to retrospectively audit published papers or to quantify the prevalence of verifiable errors within the established AI literature.

Our work departs from these prior approaches in two key ways: we concentrate specifically on identifying and fixing objective, correctness-based errors, and we apply our system to previously published NeurIPS, ICLR, and TMLR papers to systematically measure error rates at scale.

\section{AI Correctness Checker methodology}

\subsection{AI Correctness Checker pipeline}

Our Correctness Checker is composed of several modules. A first LLM works as an error detector that focuses on identifying objective mistakes like mathematical errors, logical contradictions, and miscalculation. Then, a second LLM double checks each of the mistakes and removes false positives. It also flags mistakes that are potentially more substantive, which are defined to be mistakes that could potentially change some of the paper's results, alter interpretation of the findings, or lead to non-obvious confusion for a typical reader. Examples of substantive mistakes include invalid proofs, incorrect formulas in main results, etc. Mistakes that have more limited impact, such as incorrect cross-references, wrong notations, or typos in formula that do not affect the main results, are not flagged as substantive. Both LLMs are instructed to not comment on subjective issues such as experimental design choices, writing quality, and novelty. We use GPT-5 with medium reasoning effort for both reviewers. We created an additional module using GPT-5-mini to annotate the category of each mistake: mistake in table/figure; mistake in math/formula; mistake in cross-reference; or mistake in text (see Table~\ref{tab:error_cateogry}). This pipeline takes as input the PDF and a single issue and returns one of the four possible categories. Finally, we also implemented a component to suggest possible fixes to the identified mistakes using GPT-5, which provides concrete corrections for clearly safe changes (e.g., reference errors, figure mislabeling, obvious typos) while responding with ``No immediate fix'' for issues that can not be resolved or would require substantial rewriting of the paper.

\begin{table}[t]
\centering
% 
% \resizebox{\textwidth}{!}{
\begin{tabular}{lp{12cm}}
\toprule
\textbf{Category} & \textbf{Common Examples} \\ 
\midrule
mistake in table/figure & wrong calculations in tables/figures, mistakes in captions, inconsistencies between table/figure and text  \\[5pt] 
mistake in math/formula & typos in equations, incorrect derivations or logic in proofs, contradictory assumptions \\[5pt] 
mistake in cross-reference &  incorrect references to equations, tables, or figures \\[5pt] 
mistake in text &  incorrect or logically imprecise explanation or definition  \\ 
\bottomrule
\end{tabular}
% }
\caption{The four mistake categories and their common examples. Each category is defined by where the corresponding issue appears, and each category has a fixed precedence: when an issue fits multiple categories, the one with higher precedence is applied first. The list is ordered from highest to lowest precedence.}
\label{tab:error_cateogry}
\end{table}

\subsection{Sampling of published AI papers}
NeurIPS, ICLR, and TMLR are among the most influential and popular venues in AI research, and together provide a representative snapshot of current methodological and scientific trends in the field. These conferences and the journal also share the advantage of using standardized publishing infrastructure through OpenReview\footnote{\url{http://openreview.net/}}, which allows us to reliably download manuscripts in consistent formats. In addition, papers published in these venues are single-column, a layout that is more compatible with optical character recognition (OCR) pipelines and less likely to introduce parsing errors in our experience. For this reason, we also avoid older NeurIPS papers, as PDF compilation differences in earlier years can lead to additional OCR artifacts. Further, to reduce computational load and minimize extraction failures, we exclude papers larger than 10 MB.

Using the OpenReview API, we randomly sampled and downloaded 1600 published papers from ICLR 2018–2025, 500 papers from NeurIPS 2021–2025, and 400 papers from TMLR 2022–2025. The papers are uniformly distributed over the years. We do not differentiate by presentation type (e.g., spotlight, oral, or poster). Because TMLR first started in spring 2022, we grouped TMLR papers from 2022 and 2023 together for downstream analysis. 

Most papers in these venues include appendices, but submission rules governing appendix placement have varied over time. ICLR has generally asked authors to embed appendices directly within the main PDF, while NeurIPS only mandated this practice beginning in 2024–2025. These differences can create inconsistencies in document structure and formatting. To ensure that our findings are not confounded by such policy shifts, we also present a separate analysis for NeurIPS in which we restrict our evaluation to the first 10 pages of each paper (published NeurIPS papers have up to 10 pages of main content). This provides a consistent content window across different years regardless of whether appendices are included in the main file.\footnote{For experiments limited to the first 10 pages of NeurIPS papers, we slightly adjust the evaluation prompt to note that appendices may not be included and that references may be truncated. We instruct the LLM not to penalize papers for these omissions.}

\subsection{Estimating precision and recall of the AI Correctness Checker}

To assess the reliability of the AI Correctness Checker, we measure its precision and recall using human verification. Throughout this paper, we define \emph{precision} as the proportion of issues flagged by the Checker that correspond to actual mistakes, and \emph{recall} as the proportion of all actual mistakes in a paper that are successfully detected by the AI Checker.  We elaborate on the experimental settings below.

For precision, we first randomly sample 60 published papers that contain at least one potentially substantive mistake as flagged by the AI Checker. The checker detected a total of 316 potential mistakes in these papers.
The selected papers are randomly assigned to the authors of this manuscript, and human annotators independently review every identified issue in the papers assigned to them.
The human annotators determined whether each issue is a real mistake or a false positive (i.e. the original paper is correct). They also annotated whether the mistake is potentially substantive, as defined above, and categorized the type of mistake as Math/Formula, Text, Table/Figure, or Cross-reference error. 

For recall, we consider a controlled setting in which we intentionally inject errors into several papers and measure how often the AI Checker successfully flags them. Specifically, we select five papers accepted to NeurIPS (2021–2025) or ICLR (2018–2025) that include at least one of our authors as a co-author \citep{izzo2021dimensionality, kwon2022weightedshap, yuksekgonul2023when, jiang2023opendataval, kwon2024datainf}. For each paper, we create three mistake-injected copies, each containing six mistakes. This yields a total of 15 mistake-injected papers with 90 different injected mistakes, which we then feed into the AI Checker pipeline. We run the AI Checker pipeline three independent times for each mistake-injected copy and manually review the AI Checker’s reports. Then, the average recall is computed based solely on the injected mistakes and does not account for potentially pre-existing mistakes.

The injected mistakes are distributed across different sections of each paper and are crafted to span all four error categories listed in Table~\ref{tab:error_cateogry}. They also cover a range of difficulty levels: some require little to no specialized expertise—such as miscalculations, incorrect citations, or formula typos—whereas others involve more advanced mathematical concepts or domain knowledge (e.g., the definition of a Wasserstein barycenter or the derivation of influence functions).

\begin{figure}[!h]
    \centering
    \includegraphics[width=1\linewidth]{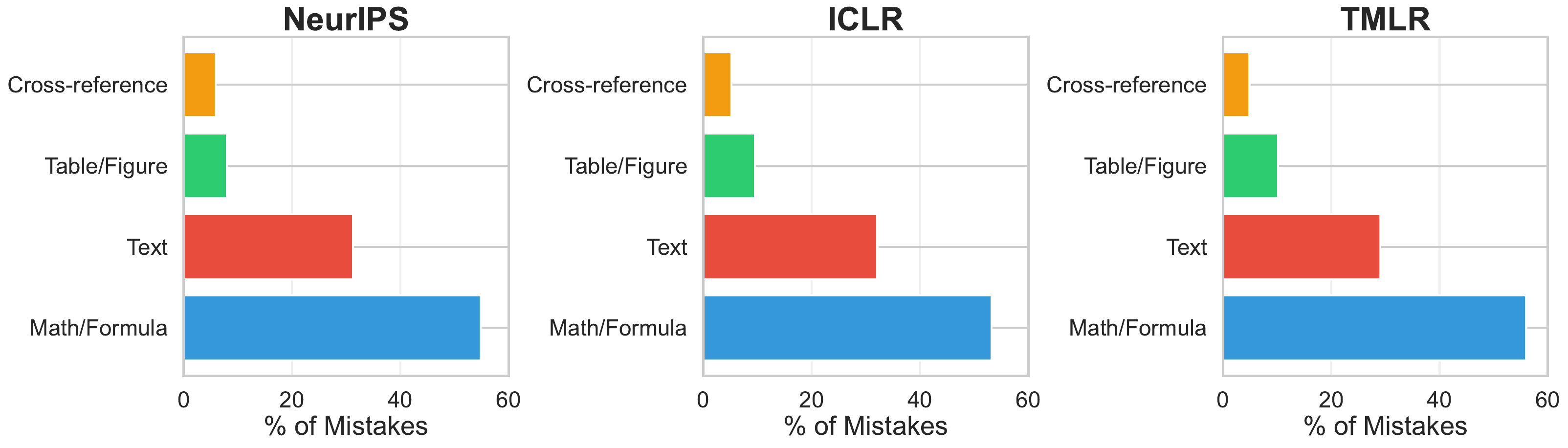}
    \caption{Percentage of mistakes by category for NeurIPS (left), ICLR (middle), and TMLR (right). The distributions of mistake types are similar across different venues, with math and formula mistakes being the most prominent category in published AI papers. We provide several representative examples the AI Checker identified in Section~\ref{sec:example}.}
    \label{fig:error_categories}
\end{figure}

\section{Results}
\subsection{Quantifying errors in published papers}

Our analysis (see Figure~\ref{fig:error_rates}) of 2,100 papers from ICLR (2018--2025) and NeurIPS (2021--2025), plus 400 papers from TMLR (2022--2025), shows that published AI research papers contain an average of $4.66$ mistakes with a standard error of $0.04$, with ICLR, NeurIPS, and TMLR averaging 4.49, 4.68, and 5.28 mistakes respectively. Overall, 2481 of the 2500 papers (99.2\%) contain at least one mistake flagged by the checker.
Notably, 23.8\% of ICLR papers, 30.8\% of NeurIPS papers, and 36.0\% of TMLR papers contain at least one potentially substantive mistake that could affect interpretation, reproducibility, or downstream use.  A clear temporal trend is also evident: the average number of mistakes in NeurIPS papers increased from 3.8 in 2021 to 5.9 in 2025 (a 55\% increase), in ICLR papers rose from 4.1 in 2018 to 5.2 in 2025 (a 27\% increase), and TMLR papers climbed from 5.0 in 2022--2023 to 5.5 in 2025 (Figure~\ref{fig:error_rates}). These results suggest that the error burden in top-tier AI papers has grown steadily over time.

The AI Checker categorized the detected mistakes into four types: Math/Formula (54.0\%), Text (31.4\%), Table/Figure (9.3\%), and Cross-reference errors (5.3\%), with similar distributions observed across all venues (Figure~\ref{fig:error_categories}). Human assessment confirmed that the AI classification of mistake categories is reliable with an accuracy of 84\%. The predominance of mathematical mistakes is particularly noteworthy, given their direct implications for correctness, reproducibility, and rigor.

To control for differences in the appendix length of papers, we conducted an additional experiment using the AI Checker to assess only the main text of each NeurIPS paper (i.e. the first 10 pages). The AI Checker also identified an increase over time in the average number of mistakes per paper and in the fraction of papers with one or more potentially substantive mistakes (Figure~\ref{fig:error_rates_nuerips_10_pages}). 

\begin{figure}
    \centering
    \includegraphics[width=0.875\textwidth]{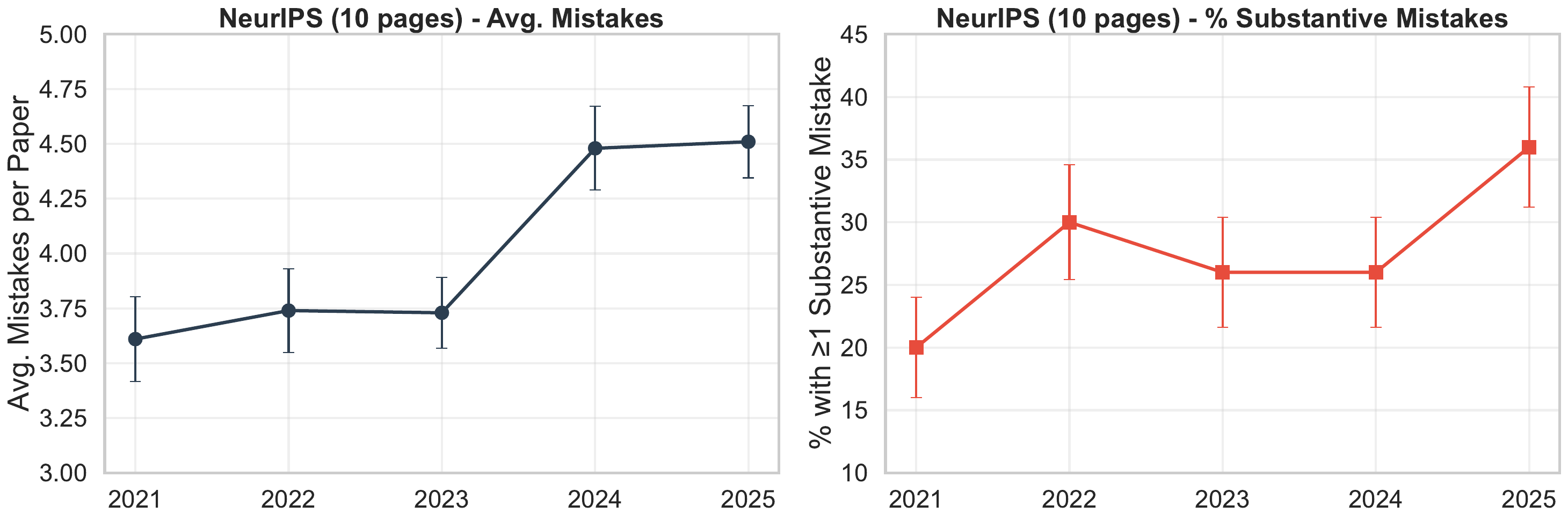}
    \caption{AI Checker detected mistakes in published NeurIPS papers when reviewing only the first 10 pages of each paper to control for paper length. The average number of detected mistakes per paper (left); The percentage of papers with $\geq 1$ potentially substantive mistake (right).  Error bars represent the standard errors. Even after accounting for paper length, the percentage of papers with at least one potentially substantive mistake have increased over time.}
    \label{fig:error_rates_nuerips_10_pages}
\end{figure}

\subsection{Precision and recall of the AI Correctness Checker}
In our validation set of 60 randomly selected papers, human researchers manually examined each of the 316 potential mistakes identified by the AI Checker and confirmed that 263 are genuine mistakes in the papers. This corresponds to a precision of 83.2\% for the AI Checker. Among the 263 confirmed mistakes, the AI Checker and human reviewers annotated a similar number as more substantive mistakes---76 by AI and 86 by human, with overlap of 62 mistakes. 

In our recall analysis, the AI Checker achieves an overall recall of 60.0\% across the 90 injected mistakes. Recall differs substantially across mistake categories: the AI Checker detects Math/Formula mistakes most reliably (66.7\%), followed by Table/Figure mistakes (61.9\%), whereas Text (55.9\%) and Cross-reference mistakes (53.8\%) remain more challenging (Figure~\ref{fig:recall-per-mistake}). This pattern indicates that the AI Checker is more effective when errors occur in structured mathematical expressions, whereas mistakes embedded in narrative text or cross-references are harder for the system to capture. The imperfect recall suggests that the AI Checker might provide a conservative estimate of the number of mistakes in papers. 

\begin{figure}[t]
    % \centering
    \begin{subfigure}{0.5\textwidth}
        \begin{center}
            \includegraphics[width=0.85\linewidth]{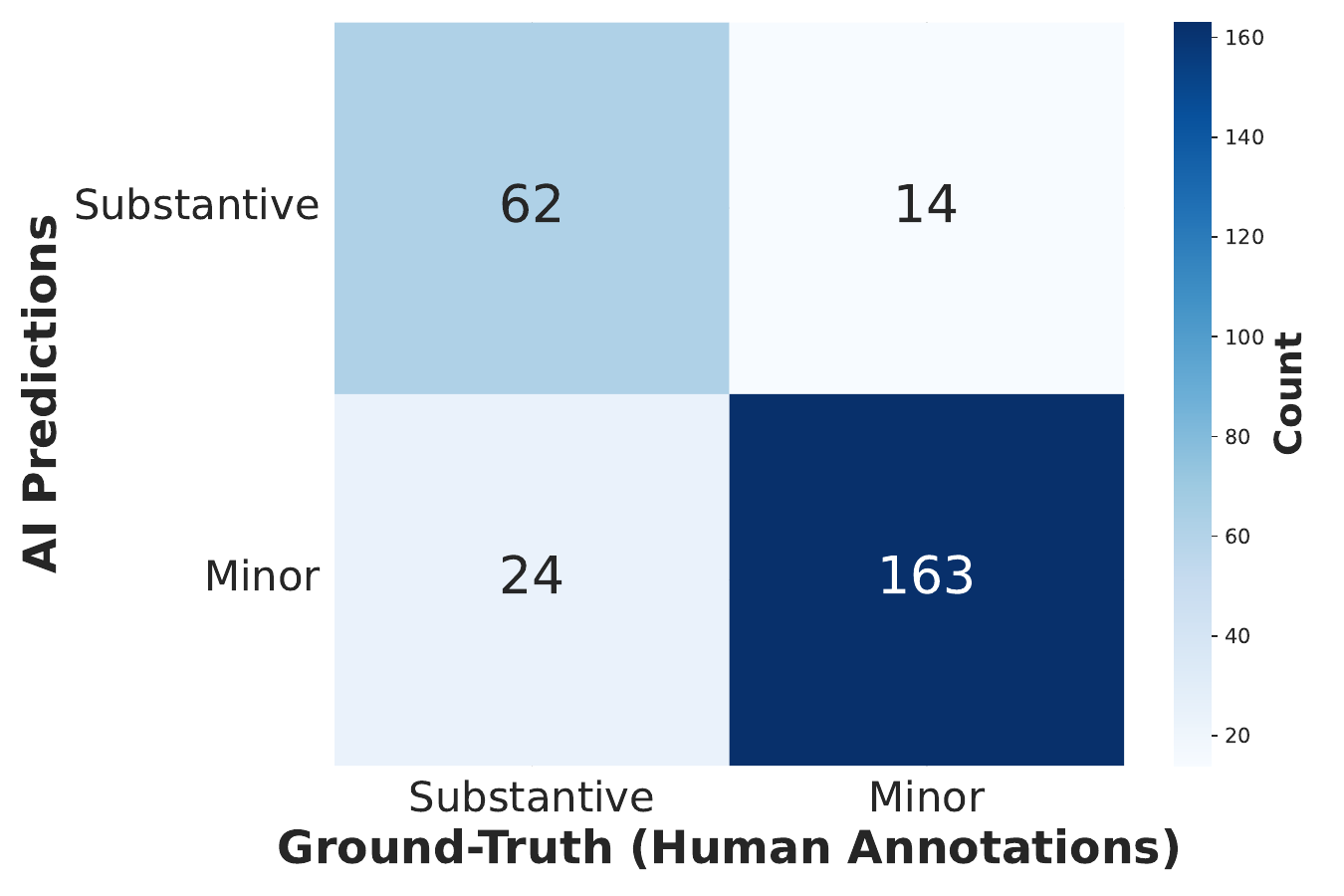}
            \caption{Contingency table of confirmed mistakes}    
            \label{fig:precision-contingency}
        \end{center}
    \end{subfigure}
    \hfill
    \begin{subfigure}{0.5\textwidth}
        \begin{center}
            \includegraphics[width=\linewidth]{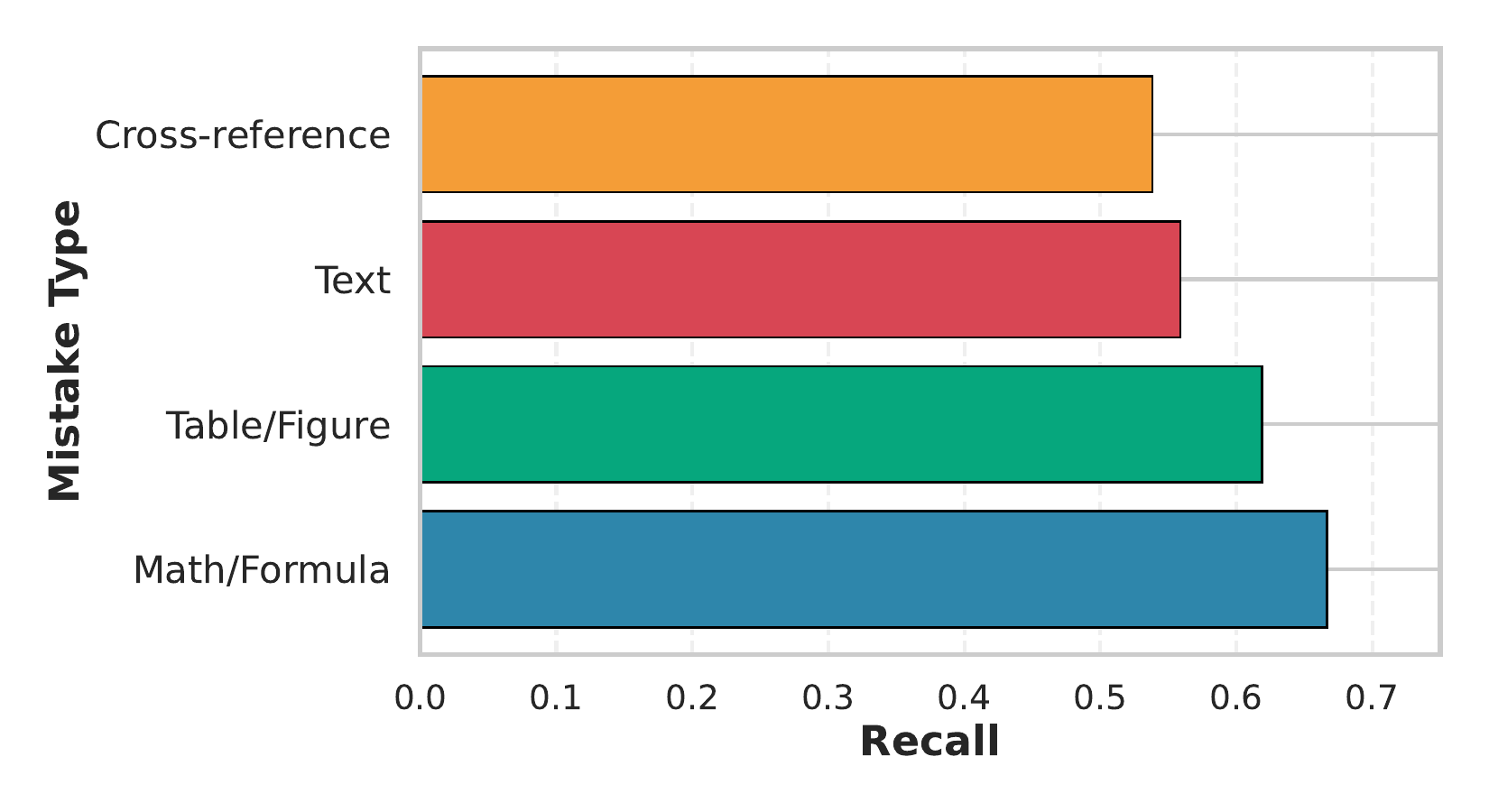}
            \caption{Recall across mistake categories}
            \label{fig:recall-per-mistake}
        \end{center}
    \end{subfigure}
    \caption{Human-verified evaluation of our AI Correctness Checker performance. Contingency table of the 263 mistakes identified our AI Checker and confirmed by humans (left) and recall across mistake categories on the 90 injected mistakes (right). Overall, the AI Checker shows relatively high precision in identifying actual mistakes---263 of 316 flagged issues are genuine mistakes. Detecting all the mistakes in a paper is more challenging. Therefore, the number of mistakes identified by the AI Checker can be interpreted a conservative lower estimate, as unflagged mistakes may still remain.}
    \label{fig:human-verification}
\end{figure}

\subsection{Examples of AI Checker-identified mistakes}
\label{sec:example}

In this section we present several representative examples of mistakes in published papers that were identified by the AI Checker. For each mistake, we provide a paraphrased excerpt from the original paper, the AI Checker's reasoning for why this is a mistake, and the human reviewer's assessment. We provide additional examples in Appendix~\ref{app:additional_examples}.

\begin{mistake}[title=Data size overstated by 100 times (mistake in Text)]
\textbf{Conference/Topic: } ICLR2018/Word Embedding\\

\textbf{Paraphrased excerpt from the original paper: }\\
The XYZ dataset includes 30 million pairs, each accompanied by an automatically estimated quality score. For model training we use only the first 15 million pairs.
\\

\textbf{AI Checker Analysis:} \\
Error: XYZ dataset size is stated as 30 million pairs and that the authors use the first 15 million for training; the XYZ corpus is $\sim$0.3 million pairs, not 30 million. Location: page 6, Dataset paragraph (“The XYZ dataset includes 30 million pairs… we use only the first 15 million pairs.”).\\

\textbf{Human Reviewer Remarks:} \\
The AI Checker Analysis is correct. The original XYZ dataset has $\sim$0.3 million pairs, not 30 million pairs. Here, AI was able to ``fact check'' authors' claims using knowledge from the relevant literature. Our fact checker pipeline did not have access to a search tool, so this knowledge must have come from training. Both the human reviewer and the AI Checker annotated this a substantive mistake because it could mislead or confuse readers.
\end{mistake}

\begin{mistake}[title=Flawed reasoning about injective functions (mistake in Math/Formula)]
\textbf{Conference/Topic : } ICLR2021/Graph Neural Networks\\

\textbf{Paraphrased excerpt from the original paper:} \\
Let $\mathcal{G}=(V,E)$ be a DAG. For a node $v\in V$, let $P(v)$ denote the parents of $v$, i.e., the nodes $u$ such that the directed edge $u\to v\in E$. Each node $v$ has initial features $x_v^0$. To compute the $\ell$-th layer representations, we proceed according to the topological order on the nodes $v$ and compute
\begin{align*}
x^\ell_v &= F^\ell\big(x^{\ell-1}_v, G^\ell(\{x^\ell_u \: | \: u\in P(v)\}, x^{\ell-1}_v)\big), \quad \ell=1,\ldots,L, \\
x_\mathcal{G} &= R(\{x^\ell_v, \, \ell=0,1,\ldots,L, \, v\in T\}),
\end{align*}
where $T\subseteq V$ denotes the set of terminal nodes (i.e., those without child nodes).

\begin{theorem}
The map $\mathcal{G} \mapsto x_\mathcal{G}$ defined by the equations above is injective if $F^\ell$, $G^\ell$, and $R$ are injective as multiset functions.
\end{theorem}
\begin{proof}
Suppose two graphs $\mathcal{G}$ and $\mathcal{G}'$ have the same representation $x_\mathcal{G} = x_{\mathcal{G}'}$. Then, from the function $R$, they must have the same target set $T$ and same node representations $x^\ell_v$ for all nodes $v\in T$ and all layers $\ell$...
\end{proof}

\textbf{AI Checker Analysis:} \\
Error: Flawed step in Theorem 1 proof—concluding from $x_\mathcal{G} = x_{\mathcal{G}'}$ and injectivity of $R$ that the two graphs have the same target set $T$ and identical per-layer node representations for each $v\in T$ across graphs. Injectivity of a permutation-invariant multiset function $R$ only implies equality of the pooled multiset, not equality of the target node sets or node-wise correspondences. The subsequent reconstruction argument depends on this unjustified identification. Location: Appendix A, Proof of Theorem 1 (Page 13).\\

\textbf{Human Reviewer Remarks:} \\
In this case, we can construct a counterexample showing that the theorem is false as stated. Thus, this is not merely a matter of a mistaken proof technique which could potentially be salvaged; rather, the conclusion is provably incorrect. \\

Consider the following two graphs: $\mathcal{G}$ and $\mathcal{G}'$ both consist of two isolated nodes  (i.e., nodes with neither children nor parents). We will refer to these nodes as $u$ and $v$. Note that in both graphs, we have $P(u) = P(v) = \emptyset$ and $T = \{u, v\}$. In $\mathcal{G}$, we define $x^0_u = 0$ and $x^0_v = 2$. In $\mathcal{G}'$, we define $x^0_u = 1$ and $x^0_v = 3$. Note that clearly $\mathcal{G} \neq \mathcal{G}'$ since the features for these two graphs are distinct.

We now consider a single-layer GNN of the form considered in the paper. Denote $F = F^1$ and $G = G^1$, and define these maps as follows:
\[
F(0, 0) = 1, \quad F(1, 1) = 0, \quad F(2, 2) = 3, \quad F(3, 3) = 2, \quad G(\emptyset, i) = i, \: i=0,\ldots, 3.
\]
Note that $F$ and $G$ defined in this manner are injective. We now compute the resulting representations on $\mathcal{G}$ and $\mathcal{G}'$. For $\mathcal{G}$, we have:
\[
x^1_u = F(x^0_u, G(\emptyset, x^0_u)) = F(0, G(\emptyset, 0)) = F(0, 0) = 1,
\]
\[
x^1_v = F(x^0_v, G(\emptyset, x^0_v)) = F(2, G(\emptyset, 2)) = F(2, 2) = 3.
\]
Thus the final representation is 
\[
x_{\mathcal{G}} = R(\{x^\ell_s \: | \: \ell\in\{0,1\}, s\in \{u, v\}\}) = R(\{x^0_u, x^1_u, x^0_v, x^1_v\}) = R(\{0, 1, 2, 3\}).
\]
For $\mathcal{G}'$, we have:
\[
x^1_u = F(1, G(\emptyset, 1)) = F(1, 1) = 0, \quad \quad
x^1_v = F(3, G(\emptyset, 3)) = F(3, 3) = 2.
\]
Thus the final representation is
\[
x_{\mathcal{G}'} = R(\{x^0_u, x^1_u, x^0_v, x^1_v\}) = R(\{1, 0, 3, 2\}) = x_{\mathcal{G}}.
\]
Therefore, this counterexample clearly shows that the map $\mathcal{G} \mapsto x_\mathcal{G}$ is not necessarily injective even if $F$, $G$, and $R$ are injective multiset functions. Both the human reviewer and the AI checker annotated this a substantive mistake.
\end{mistake}

\begin{mistake}[title=Wrong linear algebra theory (mistake in Math/Formula)]
\textbf{Conference/Topic: } NeurIPS2023/Collaborative Filtering\\

\textbf{Paraphrased excerpt from the original paper: }\\
%\textbf{Corollary 5.} [...]. 
(in Corollary 5)\\
By definition, $\Sigma_{A}$ and $\Sigma_{B}$ are positive semidefinite, therefore so is $\Sigma_{A}\Sigma_{B}$. \\

\textbf{AI Checker Analysis:} \\
Error: Misstatement that the product of covariance matrices is positive semidefinite. Corollary 5 says `By definition, $\Sigma_{A}$ and $\Sigma_{B}$ are positive semidefinite, therefore so is $\Sigma_{A}\Sigma_{B}$.' In general, $\Sigma_{A} \Sigma_{B}$ need not be positive semidefinite; only $\mathrm{det} (\Sigma_{A} \Sigma_{B}) = \mathrm{det} (\Sigma_{A} ) \mathrm{det} (\Sigma_{B})$. 
% The conclusion that $\log \mathrm{det} (\Sigma_{A} \Sigma_{B})$ goes to $-\infty$ on rank deficiency holds without this incorrect claim.
Location: Corollary 5.\\

\textbf{Human Reviewer Remarks:} \\
The AI Checker Analysis is correct. The product of two positive semidefinite matrices is not guaranteed to be positive semidefinite, and thus the claim in Corollary 5 is not correct. The AI Checker annotated this a minor mistake and the human reviewer thought this is a substantive mistake.

\end{mistake}

\begin{mistake}[title=Mistakes in probability theory (mistake in Math/Formula)]
\textbf{Conference/Topic: } ICLR2020/Optimal Control \\

\textbf{Paraphrased excerpt from the original paper:} \\
We consider an optimal control problem with states $x_t$ and actions $a_t$. We assume that the data collection policy $U$ samples actions uniformly from the action space $\mathcal{A} \subseteq \mathbb{R}^d$, and that the Lebesgue measure of $\mathcal{A}$ is $\overline{A} < \infty$.
We wish to bound the following quantity: $\mathbb{E}\left[\frac1T \sum_{t=0}^{T-1} D_{\mathrm{KL}}(P(\cdot|x_t, a_t) || Q(\cdot|x_t, a_t)) \: | \: P, x_0 \right]$ for an arbitrary control sequence $\{a_t\}_{t=0}^{T-1}$. The arbitrary action sequence can be expressed as the output of a deterministic policy $\pi'$ defined by $a_t = \pi'(x_t, t)$. Therefore, the quantity of interest can be written as:
\begin{align*}
    &\mathbb{E}\left[\frac1T\sum_{t=0}^{T-1} D_{\mathrm{KL}}(P(\cdot|x_t, a_t) || Q(\cdot|x_t, a_t)) \: | \: P, x_0\right] \\
    &=\mathbb{E}\left[\frac1T\sum_{t=0}^{T-1} \int_{a_t \in \mathcal{A}} D_{\mathrm{KL}}(P(\cdot|x_t, a_t) || Q(\cdot|x_t, a_t)) d\pi'(a_t|x_t, t) \: | \: P, x_0\right] \\
    &=\mathbb{E}\left[\frac1T\sum_{t=0}^{T-1} \int_{a_t \in \mathcal{A}} D_{\mathrm{KL}}(P(\cdot|x_t, a_t) || Q(\cdot|x_t, a_t)) \cdot \frac{d\pi'(a_t|x_t, t)}{dU(a_t)}\cdot dU(a_t) \: | \: P, x_0\right] \\
    &\leq \overline{A} \cdot \mathbb{E}_{x, a}\left[D_{\mathrm{KL}}(P(\cdot|x,a)||Q(\cdot|x,a)) \right],
\end{align*}
where we compute the expectation over the state-action occupation measure 
\[
\frac1T \sum_{t=0}^{T-1} \mathbb{P}(x_t=x,\, a_t=a|x_0, U)
\]
induced by the data-sampling policy $U$. The last inequality is due to the facts that (i) $\pi'$ is a deterministic policy, (ii) $dU(a_t)$ describes the uniform sampling policy with constant Lebesgue density $1/\overline{A}$ over all control actions, (iii) the following bounds for importance sampling factor holds: $\left| \frac{d\pi'(a_t|x_t, t)}{dU(a_t)}\right| \leq \overline{A}$.\\

\textbf{AI Checker Analysis:} \\
Error: Invalid change-of-measure with a deterministic policy—use of $d\pi'/dU$ and a uniform bound $\leq \overline{A}$ when $\pi'$ is singular w.r.t. Lebesgue measure. This renders inequality (10) unjustified and undermines the bounds used in Lemmas 1–3. Location: Appendix A.1, Eq. (10) and the derivations following it (pages 11–12); analogous dependence in Appendix A.2 (Eq. (14), page 13) and Appendix A.4 (Eq. (19), page 15). \\

\textbf{Human Reviewer Remarks:} \\
Since $\pi'$ is a deterministic policy, it places probability mass on a single point. In particular, this implies $\pi'$ is not absolutely continous with respect to a continuous uniform distribution $U$ (note $U$ must be continuous since it is stated to have a Lebesgue density $1/\overline{A}$), and thus the Radon-Nikodym derivative $\frac{d\pi'(a_t|x_t, t)}{dU(a_t)}$ does not exist. The final inequality therefore does not hold, invalidating the corresponding lemma. In this paper, this same mistake actually invalidates three out of four of the main theoretical results. Both the AI Checker and the human reviewer annotated a substantive mistake.
\end{mistake}

\begin{mistake}[title=Incorrect justification of the main objective function (mistake in Math/Formula)]
\textbf{Conference/Topic: } ICLR2022/Few-shot Learning\\

\textbf{Paraphrased excerpt from the original paper:} \\
(page 5)\\
Let $\mathcal{S}_{target}$ denote a target dataset and $\mathcal{S}_{source}$ denote a large-scale source dataset. Our objective is to optimize the parameter $\psi$. We formulate this as:
\begin{align*}
    \mathrm{argmax}_{\psi} \log p( \psi \mid \mathcal{S}_{target}, \mathcal{S}_{source}) \approx \mathrm{argmax}_{\psi} [\log p( \psi \mid \mathcal{S}_{target}, \omega) + \log p(\omega \mid \mathcal{S}_{source})]
\end{align*}
where $\omega$ represents a prior pre-trained model (i.e., trained on the source data $\mathcal{S}_{source}$; specifically, we utilize a point estimate $\hat{\omega}$ fitted on $\mathcal{S}_{source}$ (See page 15).\\
\\
(page 15)\\
Assume that $\psi \perp \mathcal{S}_{source} \mid \omega$, the derivation proceeds as follows:
\begin{align*}
 \mathrm{argmax}_{\psi} \log p (\psi | \mathcal{S}_{target}, \mathcal{S}_{source}) 
=  \mathrm{argmax}_{\psi}~ \int_{\omega} [ \log p (\psi | \mathcal{S}_{target}, \omega) + \log p (\omega | \mathcal{S}_{source}) ] d\omega  
\end{align*}

\textbf{AI Checker Analysis:} \\
Error: Incorrect replacement of log of an integral with the integral of logs in the derivation. Location: page 15: $\mathrm{argmax}_{\psi} \log p (\psi | \mathcal{S}_{target}, \mathcal{S}_{source}) = \mathrm{argmax}_{\psi} \int_{\omega} \log p( \psi \mid \mathcal{S}_{target}, \omega) + \log p(\omega \mid \mathcal{S}_{source}) d\omega$; referenced by Equation on page 5.\\

\textbf{Human Reviewer Remarks:} \\
As stated in AI Checker Analysis, the first equation of the derivation on page 15 is non‑trivial and, in general, incorrect. 
If we make an additional assumption that $\psi \perp \mathcal{S}_{source} \mid \mathcal{S}_{target}, \omega$ (contrast this with the paper's original assumption $\psi \perp \mathcal{S}_{source} \mid \omega$), then by the law of total probability, $p(\psi \: | \: \mathcal{S}_{target}, \mathcal{S}_{source})$ can be written as
\begin{align}
    p(\psi  |  \mathcal{S}_{target}, \mathcal{S}_{source}) &= \int_{\omega} p(\psi  |  \mathcal{S}_{target}, \mathcal{S}_{source}, \omega) p(\omega  |  \mathcal{S}_{target}, \mathcal{S}_{source}) \, d\omega \nonumber \\
    &= \int_{\omega} p(\psi  |  \mathcal{S}_{target}, \omega) p(\omega  |  \mathcal{S}_{source}) \, d\omega. \label{eq: domain adaptation error}
\end{align}
The second equality uses 
\[
p(\psi  |  \mathcal{S}_{target}, \mathcal{S}_{source}, \omega) = p(\psi  |  \mathcal{S}_{target}, \omega)
\]
because $\psi \perp \mathcal{S}_{source} | \mathcal{S}_{target}, \omega$, as well as 
\[
p(\omega  |  \mathcal{S}_{target}, \mathcal{S}_{source}) = p(\omega  |  \mathcal{S}_{source})
\]
because $\omega \perp \mathcal{S}_{target} | \mathcal{S}_{source}$ (recall that $\omega$ represents model parameters pre-trained on the source data $\mathcal{S}_{source}$ before any adaptation to the target domain has occurred using $\mathcal{S}_{target}$, so it can be assumed independent of $\mathcal{S}_{target}$ given $\mathcal{S}_{source}$, though we remark that this is also not stated explicitly in the paper).\\

We observe that equation \eqref{eq: domain adaptation error} is (visually) almost equivalent to the equation presented in the paper; however, to obtain the paper's stated equation, we would need to take the log of both sides of \eqref{eq: domain adaptation error} (a valid operation), then interchange the integral and the log, which is not a valid operation. Therefore the paper's claim 
\begin{align*}
 \mathrm{argmax}_{\psi} \log p (\psi | \mathcal{S}_{target}, \mathcal{S}_{source}) 
=  \mathrm{argmax}_{\psi}~ \int_{\omega} [ \log p (\psi | \mathcal{S}_{target}, \omega) + \log p (\omega | \mathcal{S}_{source}) ] d\omega  
\end{align*}
is not correctly justified. The human reviewer annotated this as a substantive mistake and the AI Checker annotated it as minor.
% that their objective $\mathrm{argmax}_{\psi}~  \log p (\psi | \mathcal{S}_{target}, \hat{\omega})$ approximates the desiderata $\mathrm{argmax}_{\psi} \log p (\psi | \mathcal{S}_{target}, \mathcal{S}_{source})$ lacks proper justification.

\end{mistake}

\begin{mistake}[title=Inconsistency between text and table (mistake in Table/Figure)]
\textbf{Conference/Topic: } ICLR2022/Robotics Simulation\\

\textbf{Paraphrased excerpt from the original paper:} \\
(in page 5)\\
For every skill pairing, we draw 1,000 random initializations and then run 500 Adam optimization steps on each initialization.\\

(in page 10)
\begin{center}
    \begin{tabular}{lll}
        \toprule
        Skill pairing using Adam \\
        \hline
        N initial seeds & 1000\\
        N iteration & 1000\\
        \bottomrule
    \end{tabular}
\end{center}

\textbf{AI Checker Analysis:} \\
Error: Inconsistent number of Adam iterations for skill pairing. The main text says "run 500 Adam optimization steps" per initialization, while Table lists "N iteration 1000". Location: page 5 vs. Table in page 10.\\

\textbf{Human Reviewer Remarks:} \\
The AI Checker Analysis is correct. The paper inconsistently reports the number of iterations for the Adam optimizer. The manuscript reports 500, while Table reports 1000. Both the AI Checker and the human reviewer annotated this a  minor mistake.

\end{mistake}

\subsection{False positives by AI Checker}
\label{sec:false_positives}
False positives produced by the AI Checker were typically due to the use of non-standard notation or OCR errors. Below, we present several representative examples. 

\begin{example}[LLM confused by non-standard notation]
One paper used the letter $Z$ to denote a marginal probability vector when describing the Wasserstein distance. The notation $Z$ was explicitly and precisely defined in the paper, yet our AI Checker flagged an error, asserting that $Z$ is not a marginal probability vector. 
We believe this incorrect identification likely arises because the literature commonly uses $p$, $P$, or $\mathbb{P}$ to denote a probability vector, while $Z$ is commonly used for a logit vector. AI systems, relying on this convention, may have been misled by the non-standard notation and flagged it as an error.
\end{example}

\begin{example}[OCR in Math/Formula]
One paper presented a probabilistic inequality that provides an upper bound on the gap between sample statistics and the corresponding population parameter, as a function of the sample size $n$. The convergence rate for this bound was shown to be $O(1/\sqrt{n})$, which is mathematically correct, but the OCR mistakenly omitted the square root, reporting the rate as $O(1/n)$, which was flagged as a mistake. 
\end{example}

\begin{example}[OCR in an Algorithm box]
In another paper, the authors used different levels of indentation to express nested ``if-else'' conditional statements in an algorithm box, as shown below.
\begin{algorithm}[h]
 \caption{Example of an algorithm that uses multiple indentation levels.} 
\begin{algorithmic}
    \If{Condition 1}
        \State Statement 1
        \If{Condition 2}
        \State Statement 2
        \EndIf
    \Else
        \State Statement 3
    \EndIf
\end{algorithmic}
\end{algorithm}
\\
The AI Checker incorrectly paired the ``else'' statement with Condition 2, i.e., assuming that Statement 3 executes when Condition 1 is met but Condition 2 fails, instead of the intended behavior of Statement 3 executing when Condition 1 fails.
\end{example}

Because of these issues, it is important to interpret the AI Checker's output as possible mistakes until they are verified by humans. Fortunately, in our experience, it is relatively easy to identify OCR issues from the AI Checker's reasoning.

\subsection{Fixing mistakes with AI Checker}
Once the AI Checker identifies a mistake, it also attempts to propose a solution to correct the issue. Among a subset of 240 mistakes that were both flagged by the Checker and verified as genuine by human reviewers, the Checker suggested fixes for 207 of those mistakes (86.3\%), while for the remaining 33 mistakes it reported “No immediate fix.” The latter cases typically involve contradictions that are difficult to resolve or issues that would require substantial rewriting of the paper. In such situations, the “No immediate fix” outcome can itself serve as a signal of the severity of the underlying problem.

Human researchers then manually evaluated each of the 207 fixes proposed by the AI Checker and confirmed that 157 of those 207 (75.8\%) were correct and adequately addressed the identified mistake. This finding is consistent with the fact that most mistakes are relatively localized. It also suggests that the AI Checker can serve as a useful assistant correcting common technical issues in research papers. We give an example of the Checker providing a correct fix to a substantive error below.

\begin{mistake}[title=Fixing incorrect proof (mistake in Math/Formula)]
\textbf{Conference/Topic: } ICLR2024/Contrastive Learning\\

\textbf{Paraphrased excerpt from the original paper: }\\
(page 5; definition)\\
We define the norm on a matrix $M \in \mathbb{R}^{h \times H}$ as $\| M\| := \sum_{1 \le j < k \le H} G(M_j - M_k) $ for some norm $G$ on $\mathbb{R}^h$.\\

(page 10; proof of Theorem 1)\\
For homogeneity, for any $\mu \in \mathbb{R}$,
\begin{align}
\|\mu M\|
&= \sum_{1 \le j < k \le H} G(\mu M_j - \mu M_k) \notag \\
&= \sum_{1 \le j < k \le H} G(\mu (M_j - M_k)) \notag \\
&= \sum_{1 \le j < k \le H} \mu\, G(M_j - M_k) \tag{7} \\
&= \mu \|M\|. \notag
\end{align}

For the triangle inequality, for any $P, Q \in \mathbb{R}^{h \times H}$,
\begin{align}
\|P + Q\|
&= \sum_{1 \le j < k \le H} G\big((P_j + Q_j) - (P_k + Q_k)\big) \notag \\
&= \sum_{1 \le j < k \le H} G\big((P_j - P_k) + (Q_j - Q_k)\big) \notag \\
&= \sum_{1 \le j < k \le H} \Big( G(P_j - P_k) + G(Q_j - Q_k) \Big)  \tag{8} \\
&= \sum_{1 \le j < k \le H} G(P_j - P_k)
 \;+\; \sum_{1 \le j < k \le H} G(Q_j - Q_k) \notag \\
&= \|P\| + \|Q\|. \notag
\end{align}

This completes the proof.\\

\textbf{AI Checker Analysis:} \\
Error: The proof of Theorem 1 contains objective mistakes: (i) homogeneity is written as $\norm{\mu M}_{1} = \mu \norm{ M}_{1}$ (missing absolute value) and (ii) the triangle property is handled by using $G(a+b) = G(a)+G(b)$ instead of the triangle inequality $G(a+b) \leq G(a)+G(b)$. Location: Appendix B, Eqs. (7)–(8) on page 10.\\

\textbf{AI Suggested Fix:} \\
Appendix (p. 10): Correct Eqs. (7)–(8) as follows.\\
- Eq. (7) (homogeneity): replace $\norm{\mu M}_{1} = \mu \norm{ M}_{1}$ with $\norm{\mu M}_{1} = |\mu| \norm{ M}_{1}$ (for all $\mu \in \mathbb{R}$). \\
- Eq. (8) (triangle property): replace $G(a+b) = G(a)+G(b)$ with $G(a+b) \leq G(a)+G(b)$. \\
No other changes are required to these equations.\\

\textbf{Human Reviewer Remarks:} \\
The AI Checker Analysis is correct, and the AI Suggested Fix directly resolves the problem. With this correction, the proof of Theorem 1 becomes factually correct.
\end{mistake}

\section{Discussion}
The rapid growth in both the pace and volume of AI publications has made it increasingly difficult for authors, reviewers, and readers to filter out all mistakes that appear in research papers. As deadlines compress and the research landscape accelerates, even diligent authors and reviewers can overlook errors that would have been caught with more time and bandwidth. In this work, we demonstrate that frontier language models can play a meaningful role in addressing this gap by assisting in the systematic identification of objective mistakes in published AI papers. Using this tool, we quantify the prevalence of such mistakes in top AI venues and observe that the number of objective errors is non-negligible—and, notably, appears to be increasing over time alongside the growth of the field, which is a cause for concern.

Most of the issues detected by the LLM are relatively minor and do not alter the main conclusions of the corresponding papers. Nevertheless, correcting these issues can meaningfully improve clarity for readers, reduce confusion during reproduction attempts, and strengthen the overall reliability of the literature. Importantly, the AI Correctness Checker also surfaced substantive mistakes that could plausibly alter results or interpretations. In our validation study of 60 papers, human experts confirmed that 263 of 316 potential issues flagged by the AI Checker are genuine mistakes in the paper. They further annotated 86 of these mistakes as substantive mistakes. The AI Checker’s ability to flag and suggest fixes to such cases at scale provides a powerful starting point for more thorough verification.

There are several limitations to the AI Checker. First, it has false positives, typically due to OCR issues and some due to incomplete reasoning. We expect these issues to decrease as the vision and document-understanding capabilities of LLMs continue to advance, but at present, flagged items should be interpreted as potential mistakes rather than definitive errors until verified by human experts. Second, the AI Checker can also fail to detect genuine mistakes, meaning that our estimates may be conservative and the true rate of errors could be higher. Finally, although our analysis focuses on objective mistakes with clearly verifiable ground truth, the perceived significance of each issue is inherently more subjective. In practice, we observe disagreements both between the AI Checker and human annotators, and among the human annotators themselves, regarding whether a given mistake should be classified as substantive. For this reason, the exact counts of substantive mistakes should be interpreted with caution. Nonetheless, we believe that the overall trends (i.e., an increasing average number of mistakes and a rising fraction of papers containing at least one substantive mistake) are robust.      

Our approach is best understood as complementing, rather than replacing, human reviewers. One practical path forward is a hybrid workflow in which an LLM correctness checker highlights potential objective mistakes, and human reviewers then focus on evaluating their validity, severity, and broader implications. This would allow reviewers to devote more of their limited capacity to evaluating the subjective dimensions of a submission (e.g., novelty, technical significance, and conceptual insight) while still ensuring that objective correctness receives systematic attention. Researchers can also use the AI Checker to detect and correct technical issues before official submission. The AI Checker is fast and inexpensive, costing less than \$0.50 per paper, making it an accessible assistant to authors and reviewers. 

Finally, correcting mistakes in the published literature serves a broader purpose. Scientific research is cumulative, and future work frequently builds atop published claims, formulas, and experimental findings. Revealing and correcting errors strengthens the foundations upon which new research is constructed, reduces the risk of misinterpretation, and enhances reproducibility. We do not view our results as diminishing the efforts of researchers or the value of peer review. Rather, they highlight the inherent complexity of modern AI papers—where mathematical reasoning, algorithmic design, and empirical evaluation intersect—and illustrate how LLM-based tools can help make both new and existing work more robust. Our findings suggest a promising role for LLM assistance in improving the rigor and reliability of the scientific record as the field continues to grow. Code for the AI Checker is available on our GitHub repository.\footnote{\url{https://github.com/togethercomputer/reviewing-agents}}

\section*{Acknowledgements}
We thank Eshaan Nichani for helpful discussions on the correctness of several proofs in the examples.

\bibliographystyle{unsrt}
\bibliography{ref}

\appendix

\section{Additional Examples of AI Checker-identified mistakes}
\label{app:additional_examples}
We provide additional examples of AI Checker-identified mistakes. The additional examples are (i) incorrect derivative, (ii) wrong cross-reference, and (iii) incorrect definition. These are some of the more common mistakes identified.

\begin{mistake}[title=Incorrect derivative (mistake in Math/Formula)]
\textbf{Conference/Topic: } ICLR2021/Federated Learning\\

\textbf{Paraphrased excerpt from the original paper: }
\begin{align}
    \ell &= \log(1+e^{-\mu}) \tag{14}\\
    d\ell &= \frac{-\mu}{1+e^\mu}d\mu;\quad \frac{\partial \ell}{\partial \mu} = \frac{-\mu}{1+e^\mu} \tag{15}
\end{align}
\\

\textbf{AI Checker Analysis:} \\
Error: Wrong derivative of the logistic loss; Appendix states $ \partial \ell / \partial \mu = - \mu/(1+e^{\mu})$ and $d \ell = (-\mu/(1+e^{\mu})) d\mu$ for $\ell = \log(1+e^{-\mu})$, which is incorrect (should be $\partial \ell/ \partial \mu = -1/(1+e^{\mu}))$. Location: Appendix, Eq. (15) and the preceding line.\\

% \textbf{AI Suggested Fix:} \\
% Appendix (page 15): correct the derivative and differential of the logistic loss. Replace the line preceding Eq. (15) from 
% \begin{align*}
%     \frac{\partial \ell}{\partial \mu} = \frac{-\mu}{1+e^\mu}
% \end{align*} to 
% \begin{align*}
%     \frac{\partial \ell}{\partial \mu} = \frac{-1}{1+e^\mu}.
% \end{align*} 
% In Eq. (15), replace $d\ell = \frac{-\mu}{1+e^\mu}d\mu$ with $d\ell = \frac{-1}{1+e^\mu}d\mu$, consistent with $\ell = \log(1+e^{-\mu})$.\\

\textbf{Human Reviewer Remarks:} \\
The AI Checker Analysis is correct. $\partial \ell/ \partial \mu = -1/(1+e^{\mu})$. Both the AI Checker and the human reviewer annotated this a substantive mistake.
\end{mistake}

\begin{mistake}[title=Wrong cross-reference (Cross-Reference mistake)]
\textbf{Conference/Topic: } NeurIPS2022/Mutual Information\\

\textbf{Paraphrased excerpt from the original paper: }\\
1‑MMI clearly outperforms the others in the scenario shown in Figure 7(a), where the shared signal is one‑dimensional. This advantage disappears in Figures 7(b)–7(d), where the common component has a higher dimensionality. In fact, Figure 7(b) demonstrates that 2‑MMI usually yields the best results because it is able to capture the richer underlying structure more effectively.\\

\textbf{AI Checker Analysis:} \\
Error: Figure reference inconsistency in the independence testing section. The text refers to 'Figures 7(a)-(d)' while the relevant multi-panel figure is 'Figure 2' with subfigures (a)-(d) on the previous page. Location: Page 7, Independence testing paragraph ('For Figures 7(a) and 7(b)... In Figures 7(c) and 7(d)...'), while the actual figure is Figure 2 on page 2.\\

\textbf{Human Reviewer Remarks:} \\
The AI Checker Analysis is correct. Figure 7 does not have subfigures (a)-(d), and the corresponding paragraph explains Figure 2, not Figure 7. Both the AI Checker and the human reviewer annotated this a minor mistake.

\end{mistake}

\begin{mistake}[title=Incorrect definition (mistake in Text)]
\textbf{Conference/Topic: } ICLR2024/Multi-Agent Reinforcement Learning\\

\textbf{Paraphrased excerpt from the original paper: }\\
An unbiased estimator is one whose value converges to the true parameter as the sample size $n$ goes to the infinity.\\

\textbf{AI Checker Analysis:} \\
Error: Misdefinition and misuse of “unbiased” estimator. Footnote states “An unbiased estimator is one whose value converges to the true parameter as the sample size $n$ goes to the infinity,” which is the definition of consistency, not unbiasedness. Moreover, the stated estimators are not unbiased for finite sample size $n$. Location: Appendix B, footnote 1 and surrounding text.\\

\textbf{Human Reviewer Remarks:} \\
The AI Checker Analysis is correct. The paper incorrectly defines an unbiased estimator; what it actually describes is a consistent estimator. Both the AI Checker and the human reviewer annotated this a minor mistake.

\end{mistake}

\end{document}